%% file: root.tex
\documentclass[letterpaper, 10 pt, conference]{ieeeconf} 
\IEEEoverridecommandlockouts
\usepackage[utf8]{inputenc}
\usepackage{cite}
\let\proof\relax
\let\endproof\relax
\usepackage{amsmath,amssymb,amsfonts,amsthm}
\usepackage{tabularx}
\usepackage{multirow}
\usepackage{booktabs}

\DeclareMathOperator*{\argmin}{arg\,min}
\DeclareMathOperator*{\argmax}{arg\,max}
\usepackage{algorithm}
\usepackage{graphicx}
\graphicspath{ {./figures/} }
\usepackage{textcomp}
\usepackage{xcolor}
\usepackage{multicol}
\newtheorem{theorem}{Theorem}[section]
\newtheorem{lemma}[theorem]{Lemma}

\newtheorem{definition}{Definition}
\newtheorem{assumption}{Assumption}
\newtheorem{proposition}{Proposition}

\newcommand{\LL}{\mathcal{L}}
\newcommand{\II}{\mathbb{I}}

\newcommand{\probP}{\text{I\kern-0.15em P}}
\newcommand\norm[1]{\left\lVert#1\right\rVert}
\newcommand{\nnum}{\nonumber}

\title{Certified Robust Control under Adversarial Perturbations}

\author{
Jinghan Yang$^1$, Hunmin Kim$^2$, Wenbin Wan$^3$, Naira Hovakimyan$^3$, and Yevgeniy Vorobeychik$^1$
\thanks{This paper has been accepted to IEEE-American Control Conference 2023. This work has been supported by the National Science Foundation (CNS-1932529), AFOSR and \#FA9550-21-1-0411, NASA 80NSSC20M0229, and UIUC STII-21-06.}
\thanks{Jinghan Yang and Yevgeniy Vorobeychik are with the Department of Computer Science \& Engineering, Washington University in St. Louis, MO.
{\tt\small \{jinghan.yang,yvorobeychik\}@wustl.edu}}
\thanks{Hunmin Kim is with the Department of Electrical and Computer Engineering, Mercer University, Macon, GA.
{\tt\small kim$\_$h@mercer.edu}}
\thanks{Wenbin Wan and Naira Hovakimyan are with the Department of Mechanical Science and Engineering, University of Illinois at Urbana-Champaign, Champaign, IL.
{\tt\small \{wenbinw2,nhovakim\}@illinois.edu}}
}

\begin{document}
%

\maketitle              
\begin{abstract}
Autonomous systems increasingly rely on machine learning techniques to transform high-dimensional raw inputs into predictions that are then used for decision-making and control. However, it is often easy to maliciously manipulate such inputs and, as a result, predictions. While effective techniques have been proposed to certify the robustness of predictions to adversarial input perturbations, such techniques have been disembodied from control systems that make downstream use of the predictions. We propose the first approach for composing robustness certification of predictions with respect to raw input perturbations with robust control to obtain certified robustness of control to adversarial input perturbations. We use a case study of adaptive vehicle co
\end{abstract}

\input{intro}

\input{relatedWork}

\input{prelim}

\input{certified}
\input{control}

\input{experiment}

\input{conclusion}

\bibliographystyle{plain}
\bibliography{references.bib}

\end{document}

%% file: intro.tex
\section{Introduction}

Traditional autonomous systems rely on highly reliable control algorithms and high quality sensory information to perform relatively narrowly defined tasks, such as vehicle autopilot~\cite{dikmen2016autonomous} and robotic assembly line control~\cite{zhu2018robot,chutima2022comprehensive}.
Increasingly, however, the notion of autonomy has broadened to involve complex behavior in broader domains, such as autonomous driving, where sensory measurements are high-dimensional, obtained using a camera and/or LiDAR~\cite{wang2018networking,Li2019AdversarialCS}.
In such domains, modern algorithmic approaches for computer vision have become critical as a means to compress complex sensory data into interpretable information that can subsequently be used in control.
In particular, transformative advances in the use of deep neural networks for common vision tasks such as image classification and object detection have enabled practical advances in problems such as autonomous driving~\cite{chen2021deep}.

Despite considerable advances, however, neural network models that are highly effective in visual prediction tasks are nevertheless also highly susceptible to small (often imperceptible) adversarial perturbations to the same inputs~\cite{chen2021deep}.
In turn, extensive literature has emerged to investigate approaches for robust machine learning~\cite{cohen2019certified,salman2019provably}, where robustness is either an empirical property (evaluated using actual techniques for generating adversarial perturbations)~\cite{wong2020fast} or can be formally verified through approaches often termed \emph{certified robustness}~\cite{chiang2020detection,cohen2019certified,lecuyer2019certified,salman2019provably}.
A common goal of certified robustness is prediction invariance: that is, what is the maximum that an input can be adversarially perturbed without changing the prediction~\cite{salman2019provably,cohen2019certified}?
As prediction invariance is only sensible in classification, its natural regression counterpart certifies a prediction interval for a specified bound on the magnitude of the adversarial perturbation~\cite{chiang2020detection}.

However, predictions are typically a means to control, and mistakes in predictions are significant because they can result in catastrophic mistakes in control, such as a crash of an autonomous car.
As such, disembodied certification on prediction properties is inherently limited.
For example, invariance is often too strict since alternative predictions may have little impact on system properties, such as safety and stability.
It is clearly crucial to couple certified robustness of predictions with control in a way that enables us to certify the natural robustness properties of controllers, such as stability.

We propose a simple approach for combining robustness certification of prediction (either classification or regression) with control by making use of robust control algorithms that leverage uncertainty sets about time-invariant dynamic system parameters as input.
This, coupled with a notion of class-conditional safety sets, enables us to obtain end-to-end certificates of controller robustness under adversarial perturbations to raw high-dimensional sensory inputs.
We then instantiate our approach in the context of vehicle lateral dynamics, obtaining a control algorithm that yields a robust controller that is composed of interval-based prediction certificates. Finally, we extensively evaluate the proposed approach for end-to-end certified robustness of composition of vision and control, demonstrating the value of the certificates.

%% file: relatedWork.tex
\section{Related Work}

The problem of adversarial perturbations to inputs has now been studied, particularly in the context of computer vision~\cite{goodfellow2014explaining,chakraborty2021survey,yang2020finding,eykholt2018robust}.
Moreover, a number of recent efforts have been devoted to developing techniques to improve the robustness of machine learning to adversarial perturbations~\cite{wang2021augmax,hendrycks2019augmix,hendrycks2021many,rusak2020simple}, with many such approaches aiming to formally certify robustness~\cite{chiang2020detection,Salman2019ProvablyRD}.
Our work blends certified robustness of perceptual reasoning with robust adaptive control.
Adaptive control, which adapts a controlled system to an uncertain environment by adjusting uncertain parameters, has been studied for a few decades~\cite{hovakimyan2010L1,sastry1990adaptive}. With the advance of machine learning, recent works expand adaptive control to learning-based control, which can learn more complex and higher dimensional functions~\cite{aswani2013provably,fisac2018general,koller2018learning}. Since the learning-based control cares about system stability and safety, it is often called a safe-learning. The common idea is to defer exploring potentially unsafe regions until after getting sufficient data. Due to this assumption, the system with learning-based controls is in danger of failure when applied to autonomous vehicles that operate in dynamically changing environments, where they cannot choose \textit{mild} and ysafe environments to explore. As a result, when they begin to learn dynamic systems in uncertain environments, they may already lose control, and it is too late to restore controllability. In terms of learning-based control, the current paper addresses the problem of those controllers' reactive nature with respect to environmental changes by incorporating vision. In particular, the proposed control system predicts an uncertain environment from look-head information and adapts to this environment in advance.

%% file: prelim.tex
\section{Preliminaries}

Consider the dynamical system of the following form:
\begin{subequations}
\label{E:dynsys}
\begin{align}
    &\dot{s}(t) = G(y,s(t),\pi(t),w,\theta,\sigma(t))\\
    &o(t) = c^T s(t)
\end{align}
\end{subequations}
where $s(t)$ is true system state at time $t$, $\pi(t)$ is controller, $y \in \mathbb{R}^m$ is a vector of real-valued parameters that influence system dynamics, $c$ is the known output matrix, $o(t)$ are measurements, and $w$, $\theta$, and $\sigma(t)$ are unknown input gain, state-dependent uncertainty, and time-varying uncertainty, respectively.
All of the uncertainties can also depend on $y$. We will discuss this later.
A common goal in robust adaptive control, such as $\LL_1$ adaptive control, is to design a controller $\pi(t)$ which yields stability in the limit and also guarantees bounded transient tracking error.
We formalize this goal as follows.
Let $\pi_{\mathit{ref}}$ be the \emph{reference} controller and $s_{\mathit{ref}}$ the \emph{reference} state, which correspond to system behavior when uncertainty is perfectly tracked during uncertainty estimation (this will be clear below when we instantiate our setting in the concrete lateral vehicle control setting).
Additionally, let $\pi_{\mathit{des}}$ and $s_{\mathit{des}}$ be the \emph{design} controller and state, respectively which are associated with ideal system behavior (i.e., where error is 0 for all $t$).
We now formalize our particular meaning of robust control here.
\begin{definition}
\label{D:robustcontrol}
A controller $\pi(t)$ is \emph{robust} if there exist positive constants $c_1$ and $c_2$ such that (1) 
$
    \|s_{\mathit{des}}-s(t)\|_{\infty} \leq c_1 \ \mathrm{\&} \
    \|\pi_{\mathit{des}}-\pi(t)\|_{\infty} \leq c_2
$
for $\forall t$, and (2) 
$
    \lim_{t \rightarrow \infty}
    \|s_{\mathit{ref}}(t)-s(t)\|_{\infty}=0 \ \mathrm{\&} \ \lim_{t \rightarrow \infty}
    \|\pi_{\mathit{ref}}(t)-\pi(t)\|_{\infty}=0.
$
\end{definition}

Our central focus is the case where uncertainty in the dynamics stems predominantly from uncertainty about $y$.
In particular, below we will consider an autonomous driving setting in which $y$ corresponds to friction (more precisely, cornering stiffness of the vehicle that results from it), and we estimate $y$ by first obtaining a high-dimensional visual input $x$ (e.g., a camera frame) through the use of a deep neural network $f(x)$.
Thus, the dynamical system is a composition of predictions mapping raw sensory inputs $x$ into parameters of system dynamics, state, and controller.
In particular, the central source of uncertainty that we are concerned about are \emph{adversarial perturbations} to the input image $x$, denoted by $\delta$, where $f(x+\delta)$ is substantively different from $f(x)$.
We consider two prediction cases: classification and regression.

A common assumption in prior literature on adversarial perturbation attacks is that all errors are equally bad~\cite{goodfellow2014explaining,madry2018towards,Vorobeychik18book}.
Consequently, common efforts on certifying robustness of predictions to adversarial perturbations is focused on prediction invariance~\cite{chiang2020detection,cohen2019certified}.
When we couple predictions $f(x)$ and dynamics and control in Equation~\eqref{E:dynsys}, however, not all errors are equally consequential (some may destabilize the system, whereas others will not significantly change stability), and some prediction errors may seem small in absolute terms, but can result in severe safety violations.
Our goal is to enable certification of \emph{robust control} to adversarial perturbations to raw sensory inputs $x$ of the system described above composed of predictions $f(x)$ and dynamics in Equation~\eqref{E:dynsys}.

It will be useful below to take advantage of the transparent semantics of parameters $y$ in the context of classification-based predictions $f(x)$ to define for each label $l \in L$ a \emph{safe} set of labels $S(l)$.
For example, if the true label is that the weather is sunny, predicting that it is rainy is ``safe'' in the sense that it would cause the controller to only be more conservative.
On the other hand, predicting that the weather is sunny on a rainy day potentially leads to unsafe behavior.

%% file: certified.tex
\section{Certifying Robustness of Control to Adversarial Input Perturbations}
\label{S:certified}

We now present our approach for certifying robustness of control of dynamical systems described in Equation~\eqref{E:dynsys} in which a function $f(x)$ (e.g., a deep neural network) uses raw perceptual inputs $x$ to predict parameters $y$ of system dynamics.
We focus attention on adversarial perturbations $\delta$ with bounded $\ell_2$ norm.
In particular, we will build on the techniques of \emph{randomized smoothing}~\cite{cohen2019certified} and \emph{percentile smoothing}~\cite{chiang2020detection} in order to obtain bounds on $\|\delta\|_2$ that guarantee that the controller is robust as formalized in Definition~\ref{D:robustcontrol} to arbitrary adversarial perturbations within these bounds.
We first consider the classification and subsequently the regression variants of the prediction problem.

\paragraph{Classification Settings} Consider a classifier $f(x)$ that outputs a label $l$  which is then mapped to a set $Y$ of possible values for system parameters $y$, and recall that for each $l \in L$, $S(l)$ is a set of safe predictions.
We now construct a \emph{smoothed} classifier $g(x)$ as follows.
Let $\gamma$ be a random variable distributed according to a zero-mean isotropic Gaussian distribution $\mathcal{N}(0,v^2 I)$, where $I$ is the identity matrix and $v^2$ the variance (which we would specify exogenously to balance the tradeoff between performance and robustness).
Then
\(
g(x) = \argmax_{l'} \mathbb{P}\{f(x+\gamma) = l'\}
\)
is the smoothed counterpart of $f(x)$ for each input $x$, where the probability is with respect to $\gamma$.
In practice, we estimate $g(x)$ by Monte-Carlo sampling~\cite{cohen2019certified}.
The next result is a direct adaptation of prior results certifying robustness of $g(x)$ to allow us to  consider safe sets of labels $S(l)$.
Proposition \ref{lemma:classification} gives the robust function $g(x)$ a certificate in terms of the strength of the adversarial perturbation. 
If the additive corruption to the input is within this certificate, 
the smoothed function guarantees its prediction of this adversarial input is within the safe set. 
\begin{proposition}
Let 
$a = \argmax_{a \in L} g(x)$
and
$b = \argmax_{b \in {L \setminus S(a)}} g(x)$.
Then $g(x+\delta) \subseteq S(a)$,
for all $\delta$ such that $\norm{\delta}_{2}\leq \tau$,
where $\tau = \frac{v}{2}(\Phi^{-1}(\mathbb{P}_a)-\Phi^{-1}(\mathbb{P}_{b}))$,and  $\mathbb{P}_a = \mathbb{P}(f(x+\gamma)=a)$,$\mathbb{P}_b = \mathbb{P}(f(x+\gamma)=b)$.
\label{lemma:classification}
\end{proposition}

\begin{proof}
Using the Lipschitz continuity result~\cite{cohen2019certified}, we have
\[
\norm{\Phi^{-1}(f(x)_{a}) - \Phi^{-1}(f(x+\delta)_{a_i})} \leq \frac{1}{v} \norm{\delta}_2
\]
where $a_i \in S(a) \setminus a$. 
For an adversary $\delta$, $f(x+\delta)_b \geq f(x+\delta)_{a_i}$, for some class $b \in L\setminus S(a)$, 
\begin{equation}
\norm{\delta}_2 \geq \frac{v}{2}(\Phi^{-1}(\mathbb{P}_{a_i}) - \Phi^{-1}(\mathbb{P}_b)) 
\label{equ:proof}
\end{equation}
$\forall a_i \in S(a)$, the above equation gives a lower bound on the minimum $l_2$ adversarial perturbation required to flip the classification from any $a_i$ to $b$. 
We know that the bound is minimized when $\mathbb{P}_b$ is maximized over the set of classes $L \setminus S(a)$. 
In order to have the prediction not be any of the class in set $S(A)$, we should have inequality \eqref{equ:proof}, $\forall a_i \in S(A)$.
Therefore $\norm{\delta}_2$ should be bigger than when $\mathbb{P}_{a_i}$ is maximized over the set of classes $S(a)$. 
\end{proof}

We use Proposition~\ref{lemma:classification} combined with conventional robust control to provide the end-to-end robustness guarantee.
First, we define what we mean by a robust control algorithm.
\begin{definition}
\label{D:robustcontrolalg}
Suppose that $\mathcal{A}$ is a control algorithm that takes as input a specification~\eqref{E:dynsys} of a dynamical system and a set $Y$ such that the true system parameters $y \in Y$.
We say that $\mathcal{A}$ is \emph{robust} if it returns a robust policy $\pi(t)$.
We use $\mathcal{A}(Y)$ to explicitly indicate that $\mathcal{A}$ takes the set $Y$ as input.
\end{definition}
We will discuss a particular robust adaptive control method for vehicle lateral dynamics.
The next key result follows by the definition of a robust control algorithm and Proposition~\ref{lemma:classification}.
\begin{theorem}[Classification Setting]
\label{T:robustcontrol-classification}
Suppose that $y \in \zeta(g(x))$ (i.e., $g(x)$ produces a prediction, and maps $\zeta$ to system parameters) and let $\mathcal{A}$ be a robust control algorithm.
Then $\mathcal{A}(\zeta(g(x+\delta)))$ is robust for any $\delta$ such that $\norm{\delta}_{2}\leq\tau$, where $\tau$ is as defined in Proposition~\ref{lemma:classification}.
\end{theorem}

In the adversarial setting, if the malicious corruption in the environment is within the certified radius, the predicted y system dynamics parameters from the robust model $g$ with input image $x$ is within the safe range. The control algorithm $\mathcal{A}$ thus returns a robust policy. 

\paragraph{Regression Settings} Consider now a case in which $f(x)$ is a regression.
Since we can treat each coordinate of $y$ independently, we will assume that $y$ is a scalar (i.e., a single parameter of system dynamics).
Let $\gamma$ again be zero-mean isotropic Gaussian noise as above, and define 
\begin{align}
    h_p(x) = inf\{ y\in \mathbb{R}| \mathbb{P}\Big(f(x+\gamma)\leq y \Big) \geq p \}.
\end{align}
At the high level, $h_p(x)$ is the $p$th percentile of the distribution of values of $y=f(x+\gamma)$.
We will use the \emph{median} of this distribution as our smoothed regression prediction, which we denote by $h^*(x) \equiv h_{0.5}(x)$. We make use of the following result due to Chiang et al.~\cite{chiang2020detection}:
\begin{proposition}[\cite{chiang2020detection}]
\label{lemma:regression}
For any $\epsilon$ and $\|\delta\|_2 \le \epsilon$,
\begin{align}
    h_{\underline{p}}(x) \leq h_p(x+\delta) \leq h_{\overline{p}}(x),
\end{align}
where $\underline{p}:=\Phi(\Phi^{-1}(p)-\frac{\epsilon}{v})$ and
$\overline{p}:=\Phi(\Phi^{-1}(p)+\frac{\epsilon}{v}) $.
\end{proposition}
In particular, if $\underline{p}:=\Phi(\Phi^{-1}(0.5)-\frac{\epsilon}{v})$ and
$\overline{p}:=\Phi(\Phi^{-1}(0.5)+\frac{\epsilon}{v}) $, then $h^*(x+\delta) \in [h_{\underline{p}}(x),h_{\overline{p}}(x)]$ for any adversarial perturbation $\delta$ with $\|\delta\|_2 \le \epsilon$.
We can again make use of this to obtain the following key result:
\begin{theorem}[Regression Setting]
\label{T:regression}
Suppose that $|h^*(x)-y| \le \beta$, where $y$ is the true parameter value given input $x$.
Let $\underline{y} = \min\{h_{\underline{p}}(x),h^*(x)-\beta\}$ and $\overline{y} = \max\{h_{\overline{p}}(x),h^*(x)+\beta\}$,
where $\underline{p}:=\Phi(\Phi^{-1}(0.5)-\frac{\epsilon}{v})$ and
$\overline{p}:=\Phi(\Phi^{-1}(0.5)+\frac{\epsilon}{v}) $.
Then for any $\epsilon > 0$,  $\mathcal{A}([\underline{y},\overline{y}])$ is robust for any $\delta$ with $\|\delta\|_2 \le \epsilon$.
\end{theorem}
The result follows since the conditions in the theorem ensure that the true parameters $y \in [\underline{y},\overline{y}]$.
What is particularly surprising is that this holds true for \emph{an arbitrary} $\epsilon$---that is, adversarial perturbations of arbitrary magnitude.
The reason that arbitrary perturbations cannot destabilize the system $\mathcal{A}(\zeta(g(x+\delta)))$ is that
although the perception of the environment can be maliciously modified, the robust perception model $g$ still yields a certified interval that contains the true system dynamics parameter at the current state. 
The downstream control algorithm $\mathcal{A}$ thus always returns a stable control policy. 
While this is so, higher levels of $\epsilon$ entail looser intervals $[\underline{y},\overline{y}]$, which in turn means degraded controller performance accordingly (e.g., the vehicle stops).

%% file: control.tex
\section{Certified Robust Vehicle Control}

\subsection{Vehicle Lateral Dynamics}
The current section describes the model for~\eqref{E:dynsys} on which the paper relies and the control goal.




\paragraph{Dynamic model} We use the bicycle model~\cite{rajamani2011vehicle} to model the vehicle longitudinal dynamic for lateral position $q^y$ and yaw angle $q^\psi$.
Given longitudinal velocity $V$, desired lateral position $q^{y,des}$, and desired yaw angle $q^{\psi,des}$, the differential equation of the bicycle model can be expressed as the error dynamics ((2.45) in~\cite{rajamani2011vehicle}):
\begin{align}
    \dot{s} = A s + b \pi + g \dot{q}^{\psi,des}, 
    \label{eq:dynamic}
\end{align}
where $s = [s_1,\dot{s}_1, s_2,\dot{s}_2]^\top$, $s_1 = q^y - q^{y,des}$ and $s_2 = q^\psi - q^{\psi,des}$ are the error states, $\dot{q}^{\psi,des} = \frac{V}{R}$ is the rate of the desired yaw angle, and $R$ is the radius of the road. Control input $u=d$ represents front steering angle.
The system matrices are
{\small
\begin{align*}
    &A=\left[
    \begin{array}{cccc}
        0& 1& 0& 0\\
        0& -2\frac{C_f+C_r}{m V}& 2\frac{C_f+C_r}{m}& 2\frac{-C_f\ell_f+C_r\ell_r}{m V}\\
        0& 0& 0& 1\\
        0& -2\frac{C_f\ell_f-C_r\ell_r}{I_z V}& 2\frac{C_f\ell_f-C_r\ell_r}{I_z}& -2\frac{C_f\ell_f^2+C_r\ell_r^2}{I_z V}
    \end{array}
    \right]\nonumber\\
    &     b =\left[
    \begin{array}{c}
        0\\
        \frac{2C_f}{m}\\
        0\\
        \frac{2 C_f\ell_f}{I_z}\\
    \end{array}
    \right]
    \quad
    g =\left[
    \begin{array}{c}
        0\\
        -2\frac{C_f\ell_f - C_r\ell_r}{m V} - V\\
        0\\
        -2\frac{C_f\ell_f^2 + C_r\ell_r^2}{I_z V}\\
    \end{array}
    \right]
\end{align*}}\noindent
where $m$ is the vehicle mass and $I_z$ is the yaw moment of inertia, $\ell_f$, $\ell_r$ are the front/rear tire distance from the center of gravity, and $C_f$, $C_r$ are front/rear cornering stiffness. 
Matrices $A$ and $g$ depend on velocity $V$, and $A$, $b$, and $g$ depend on cornering stiffnesses $C_f$ and $C_r$. The cornering stiffness $C_f$ and $C_r$ have a linear relation $F_f = C_f \nu$ with respect to the lateral force $F_{f}$ for a small sliding angle $\nu$.

\paragraph{Uncertainty model}
The cornering stiffnesses $C_f$ and $C_r$ are the road parameters where the vehicle is driving. Thus it is reasonable to assume that they are time-varying and unknown in advance.
Consequently, we obtain them by predicting road friction from raw sensory inputs $x$.
However, we aim to ensure the robustness of control to adversarial perturbations $\delta$ to raw inputs $x$, and the resulting prediction error induces uncertainty in the dynamic model~\eqref{eq:dynamic}.
Henceforth, to simplify discussion we assume $C_f = C_r \equiv C$.

\paragraph{Control objective}
We aim to stabilize the error state $s$ in~\eqref{eq:dynamic} so that the vehicle can keep the desired center lane despite adversarial perturbations to raw sensory inputs $x$.

\subsection{$\LL_1$ Adaptive Control Design}\label{sec:L1}
The key control challenge is that the system matrices in the lateral error dynamic~\eqref{eq:dynamic} are unknown because they are subject to unknown and time-varying cornering stiffness $C$.
Instead, we observe raw camera input $x$ that provides indirect and potentially noisy information about $C$, using two approaches for predicting $C$: 1) classification and 2) regression.
In the classification variant, we have a model $f(x)$ that predicts discrete properties of the scene captured by a camera, such as weather or road surface type.
In addition, each predicted class $l$ is associated with a cornering stiffness (friction) interval $[\underline{C}_l,\overline{C}_l]$.
In regression, our model $f(x)$ directly predicts road cornering stiffness, i.e., $C =f(x)$.

To induce provable robustness to adversarial perturbations, rather than using $f(x)$ directly for predictions, we apply randomized smoothing in the case of classification, obtaining a smoothed function $f(x)$, or median smoothing in the case of regression, obtaining $h^*(x)$.
As discussed in Section~\ref{S:certified}, these can be associated with either a safe prediction set $S(l)$ and associated certification radius for classification or a certified interval for $h^*(x)$.
In either case, the procedure yields an uncertainty interval $[\underline{C},\overline{C}]$ for cornering stiffness.

To deal with the control problem in the presence of uncerainty about cornering stiffness, we will utilize $\LL_1$ adaptive controller~\cite{hovakimyan2010L1} that can rapidly compensate the impact of uncertainties within the designed filter bandwidth of it, and guarantee transient tracking error even when unknown parameters are changing. In what follows, we will explain controller design procedure in detail.

\subsubsection{Nominal Model}

The first step is to transform the model~\eqref{eq:dynamic} into a nominal model, where we will move any uncertainties out of the system matrices. As a result, the nominal system matrices are known, and have desired system properties including stability. We will then design the $\LL_1$ adaptive controller whihc forces the system~\eqref{eq:dynamic} to behave like the nominal model by canceling out the uncertainties. 

Recall that our prediction models (either classification or regression) yield an uncertainty interval for cornering stiffness.
The key assumption we make about this interval is that it includes both the true and predicted (\emph{nominal}) values:
\begin{assumption}
\label{A:stiffness}
The control algorithm takes as input an interval $[\underline{C},\overline{C}]$ such that $C, \hat{C} \in [\underline{C},\overline{C}]$, where $C$ is the true and $\hat{C}$ nominal cornering stiffness.
\end{assumption}
If we take $\pi(t) = -k_m s(t) + \pi_{ad}(t)$, the system~\eqref{eq:dynamic} can then be transformed into the following nominal model:
\begin{align}
    \dot{s}(t) &= A_m s(t) + b_m (w \pi_{ad}(t)+ \theta^\top s(t)+ \sigma(t)) \nnum\\
    o(t)&=c^\top s(t) \quad \quad x(0)=x_0,
    \label{eq:l1dynamic}
\end{align}
where $A_m = A(\hat{C},V) - k_m s$ is Hurwitz, and $b_m = b(\hat{C})$.
The gain $k_m$ will be determined later. 
The unknown parameters $w$, $\theta$, and $\sigma(t)$ are induced by the uncertainty about cornering stiffness $C$.

\subsubsection{Adaptive Controller Design}

In order to obtain both system stability and bounded transient error, we design an adaptive controller $\pi_{ad}(t)$ in~\eqref{eq:l1dynamic} that aims to cancel out the residual uncertainty $w \pi_{ad}(t)+ \theta^\top s(t)+ \sigma(t)=0$ stemming from uncertainty about $C$. 
Adaptive controller $\pi_{ad}(t)$ consists of state predictor, adaptation law, and low-pass filter as described below. The state predictor is designed using the known parts of the dynamic system in~\eqref{eq:l1dynamic} and the states of uncertainties:
\begin{align*}
    \dot{\hat{s}}(t) &= A_m\hat{s}(t)+b_m(\hat{w}(t)\pi_{ad}(t)+\hat{\theta}^\top s(t)+\hat{\sigma}(t))\nnum\\
    \hat{y}(t)&=c^\top\hat{s}(t), \quad \quad  \hat{s}(0)=\hat{s}_0.
\end{align*}
We design the adaptation law to estimate uncertainties:
\begin{align}
    \dot{\hat{w}}(t)&=\Gamma Proj(\hat{w}(t),-\tilde{s}^\top(t)Pb_m \pi_{ad}(t)) &&\hat{w}(0)=\hat{w}_0\nnum\\
    \dot{\hat{\theta}}(t)&=\Gamma Proj(\hat{\theta}(t),-\tilde{s}^\top(t)Pb_ms(t))&&\hat{\theta}(0)=\hat{\theta}_0\nnum\\
    \dot{\hat{\sigma}}(t)&=\Gamma Proj(\hat{\sigma}(t),-\tilde{s}^\top(t)Pb_m)&&\hat{\sigma}(0)=\hat{\sigma}_0,
    \label{eq:adaptive}
\end{align}
where $\Gamma>0$ is an adaptation gain, $\tilde{s}(t)=\hat{s}(t)-s(t)$ is the prediction error, and $Proj(\cdot,\cdot)$ is the projection operator defined in Definition B.3 in~\cite{hovakimyan2010L1}. Symmetric positive definite matrix $P$ is the solution of the algebraic Lyapunov equation $A_m P + P A_m^\top = -Q$, given a symmetric positive definite $Q$.

Adaptive control is designed using the adaptation states in~\eqref{eq:adaptive} as follows:
\begin{align}
    \pi_{ad}(s) = -k D(s)(\hat{\eta}(s)-k_gr(s)),
    \label{eq:filterinput}
\end{align}
where $r(s)$ is the reference signal in the Laplacian form, and $D(s) = 1/s$ is a strictly proper transfer function that forms stable low-pass filter $F(s) = \frac{wkD(s)}{1+wkD(s)}$. The gain $k>0$ is constant, and $k_g = -1/(c^\top A_m^{-1}b_m)$. The signal $\hat{\eta}(t)$ is obtained by $\hat{\eta}(t)=\hat{w}(t)\pi_{ad}(t)+\hat{\theta}^\top(t)s(t)+\hat{\sigma}(t)$.

\subsubsection{Design Control Parameters}

Now we design control parameters $\Gamma$, $k_m$, $P$, $V$, $k$, such that the proposed control input $\pi(t) = -k_m s(t) + \pi_{ad}(t)$ guarantees desired performance and robustness of the lateral state $x$ in~\eqref{eq:dynamic}.

We need to define the desired system behavior. Let us denote $s_{\mathit{ref}}$, $\pi_{\mathit{ref}}$ non-adaptive version control, i.e., the system behavior when~\eqref{eq:adaptive} tracks the uncertainty perfectly. However, the control input cannot satisfy $w \pi_{ad}(t)+ \theta^\top s(t)+ \sigma(t)=0$ because the perfect control input is filtered in~\eqref{eq:filterinput} before the implementation. Let us denote $s_{\mathit{des}}$ and $\pi_{\mathit{des}}$ the design system having the ideal system behavior such that $w \pi_{ad}(t)+ \theta^\top s(t)+ \sigma(t)=0$ holds for $\forall t$.
Using the above definition, we can say that the system well-behaves if $\|s(t)-s_{\mathit{des}}(t)\|$ and $\|\pi(t)-\pi_{\mathit{des}}(t)\|$ are small enough.

We can choose an arbitrary large adaptation gain $\Gamma>0$ so that the system performs arbitrarily close to the reference system ($s_{\mathit{ref}}(t)$ and $\pi_{\mathit{ref}}(t)$) by Theorem 2.2.2 in~\cite{hovakimyan2010L1} without sacrificing robustness, where the reference system refers the $\LL_1$ adaptive controller without adaptation. Then, the performance of the system is rendered as the error between the reference system and the design system ($\|s_{\mathit{ref}}-s_{\mathit{des}}\|_{\infty}$ and $\|\pi_{\mathit{ref}}-\pi_{\mathit{des}}\|_{\infty}$), where the design system is the ideal system that does not depend on the uncertainties.

Since $A_m$ in~\eqref{eq:l1dynamic} must be Hurwitz and $A_m(V) P + P A_m^\top(V) <0$ should hold, we choose $k_m$ and $P$ such that $A_m(V)$ is Hurwitz and $A_m(V) P + P A_m^\top(V) <0$ holds for all $V_{\min} \leq V \leq V_{\max}$, where $V_{\max} \geq V_{\min}\geq0$ are the maximum and minimum velocity of the area.

Finally, we design $V$ and $k$ together balancing performance and robustness as follows:
\begin{align}
    &\max_{k,V \in [V_{\min},V_{\max}]} V\nnum\\
    &s.t. \ \|G(s)\|_{1} \leq \lambda_{gp}, {\rm \ for \ } \forall w \in \Omega\nnum\\
    &\quad \ k \leq \bar{k}
    \label{eq:prog}
\end{align}
for constants $\bar{k}>0$, and $\lambda_{gp}<\frac{1}{L}$, where $G(s) = H(s)(1-F(s))$, $H(s)=(s\II-A_m)^{-1}b_m$ and $L=\max_{\theta \in \Theta}\|\theta\|_1$.
The first constraint refers minimum performance guarantee and the second constraint indicates a minimum robustness guarantee, where $r$ is the certified radius obtained by the classifier.
By increasing $k$, one can render $\|G(s)\|_{1}$ arbitrary close to zero and this improve the performance $\|s_{\mathit{ref}}-s_{\mathit{des}}\|_{\infty}$ and $\|\pi_{\mathit{ref}}-\pi_{\mathit{des}}\|_{\infty}$ (Lemma 2.1.4 in~\cite{hovakimyan2010L1}). However, the time delay margin decreases as $k$ increases. It is worth noting that the problem~\eqref{eq:prog} is always feasible with $V=0$.

The following result shows that the control algorithm we thus constructed (with the design parameters as chosen above) is robust in precisely the sense of Definition~\ref{D:robustcontrolalg}.

\begin{theorem} (Robust Control Pipeline)
Given a perturbed sensory input $x+\delta$, if $\delta$ is within a given certificate $\tau$, the robust model $g$ returns a robust prediction such that the corresponding cornering stiffness interval $\zeta(g(x+\delta))$ includes the true and nominal cornering stiffness. Assumption \ref{D:robustcontrol} holds. Therefore there exists positive constants $c_1$ and $c_2$ such that the constraints in definition \ref{D:robustcontrol} are satisfied , thus the \emph{end to end} pipeline $\mathcal{A}(\zeta(g))$ is robust per definition
\ref{D:robustcontrolalg}.
\label{the2}
\end{theorem}

\begin{proof}
The controller with the system satisfy $\LL_1$ adaptive control assumptions, and thus by Theorem 2.1.1 and Lemma 2.1.4 in~\cite{hovakimyan2010L1}, the statement holds true. Constant bounds $c_1$ and $c_2$ are found in~\cite{hovakimyan2010L1}.
Consequently, we can combine this robust control algorithm with both classification-based and regression-based approaches described in Section~\ref{S:certified} to obtain provably robust control algorithms under adversarial perturbations to raw sensory inputs.
In other words, we can now directly apply our main results, Theorem~\ref{T:robustcontrol-classification} in the case of classification-based cornering stiffness prediction and Theorem~\ref{T:regression} when we use regression.
\end{proof}

%% file: experiment.tex
\section{Experiments}
In this section, we empirically study the robustness of the  robust driving system described above with and without proposed formal 
end-to-end robustness certification across different weathers and road types, comparing the vulnerability of the non-robust driving system.
We conduct experiments on three datasets, including driving frames from the Carla simulator \cite{dosovitskiy2017carla} 
as well as the physical world (Road Traversing Knowledge (RTK) \cite{rateke:2020.3}, 
robotCar \cite{houts2020ford}). 
These datasets contain driving frames across four types of weather: sunny, light rain, heavy rain, and snow, and three different road surfaces: asphalt, cobblestone, and sand (in descending order of friction). 
In particular, 
\emph{Carla} contains images across three weathers, light rain, heavy rain and sunny. Each weather has 4000 images. 
\emph{RTK}  contains different road surface types: asphalt, cobblestone and sand. This dataset contains 400 frames for each road type.
\emph{RobotCar} dataset captures many different combinations of weather, traffic , and pedestrians and contains three different kinds of weathers, sunny, rain, and snow. Each weather has 2000 images. 
\noindent\begin{table*}[t]
    \small
    \centering \vspace{0.24cm}
    \begin{tabular}{c|c|c|c|c|c|c|c}
        \toprule
        Weather & Sunny & Light Rain  & Heavy Rain  & Snow & Asphalt & Cobblestone & Sand\\
        \midrule
        Road Friction  &  80k-120k   & 60k-80k   & 40k-60k & 20k-40k & 40k-60k & 40k-60k & 30k-45k\\
        \bottomrule
    \end{tabular}
    \caption{Ground truth cornering stiffness (k=1000). 
    The table is for the asphalt road type in different weathers and different road types in the dry road condition.}\vspace{-1cm}
    \label{tab:cornering2}
\end{table*}
\noindent\begin{table}[h]\small
    \centering
    
    \begin{tabular}{c|c|c|c}
    
    \toprule
             & \emph{Carla} & \emph{RTK} & \emph{RobotCar} \\
    \midrule
    Accuracy &   $98.6\%$  &   $94.2\%$   & $95.6\%$ \\
    Instability &  0.00 &   0.00   & 0.00 \\
    Velocity    &  29.42 &  28.45  & 28.46 \\
    \bottomrule
    \end{tabular}
    \caption{Road condition classification: Accuracy and performance without malicious attacks}\vspace{-1cm}
    \label{tab:class_accuracy}
\end{table}
We use \emph{cornering stiffness} to define road friction for lateral dynamic control. 
Typical cornering stiffness ranges from $20000 - 120000 N/rad$, depending on many parameters such as road condition, rim size, and inflation pressure \cite{gillespie1992fundamentals}.
In our experiments, the range of cornering stiffness, as a function of road type or weather condition, is given in Table \ref{tab:cornering2}. 

Recall that the vision-based \emph{perception-control} system has a perception model and a control algorithm $\mathcal{A}$. 
The input to the perception model is a driving frame, and the perception model's output is a predicted cornering stiffness interval. 
This predicted cornering stiffness range is the input to $\mathcal{A}$.
We, once given this range, then decide the maximum safe velocity and control parameters.
If this upper bound is too high (i.e., exceeding the true safe velocity), the vehicle may drive dangerously or crash. 
For example, if the vehicle drives at high speed on a snowy day, it may crash into other cars due to the poor driving conditions (i.e., the low friction induced by the snowy weather).
If this upper bound is too low, the car may drive inefficiently. 
For example, the vehicle drives inefficiently if it drives extremely slow on a sunny day in which diving conditions are good. 

We consider two types of attacks by the attacker's objective:
(1) increasing the velocity, 
(2) decrease the velocity. 
In the first case, the attacker decreases the stability.
For example, the malicious perturbation may increase the predicted corner stiffness, causing the car to drive at high and unsafe speeds,
e.g., a \emph{Snowy} driving frame  may now be predicted as \emph{Sunny}, causing the car to drive at higher speeds and crash into other vehicles.
Alternatively, the attacker decreases the efficiency of the car by decreasing the velocity, 
e.g., a \emph{Sunny} driving frame may now be predicted as \emph{Snowy}, causing the car to drive at lower speeds.
We will refer these two types of attacks as \emph{Stability Attack} (SA) and \emph{Efficiency Attack} (EA). 
From the optimization perspective, these two types of attacks differ in objectives. 
The objective of \emph{Stability Attack} is maximizing corner stiffness prediction:
\begin{align}
    \argmax_{\delta}f(x+\delta) .
\end{align}
The objective of \emph{Efficiency Attack} is minimizing corner stiffness prediction:
\begin{align}
    \argmin_{\delta}f(x+\delta) .
\end{align}
We consider velocity and instability as the car's performance measurements. 
Specifically, velocity is the maximum safe velocity from $\mathcal{A}$, and 
Instability implies control system instability in Lyapunov sense. Intuitively speaking, a dynamic system is Lyapunov stable if it starts near an equilibrium point (center lane) and its trajectory stays near the equilibrium point forever.
The higher the speed, the more efficient the car. 
The lower the instability, the more stable a vehicle is.\noindent\begin{table*}[h]\small
    \centering \vspace{0.24cm}
    \begin{tabular}{c|c|c|c|c|c|c}
    \toprule
    Noise & 
    \multicolumn{2}{c}{Carla} & 
    \multicolumn{2}{c}{RTK} & 
    \multicolumn{2}{c}{RobotCar} \\
    \cmidrule{2-7}
    
    $\sigma$   
    & \emph{Velocity} & \emph{Certificate}  
    & \emph{Velocity} & \emph{Certificate} 
    & \emph{Velocity} & \emph{Certificate} \\
    \midrule
    0.25 & 6.36 & 0.61 & 18.99 & 1.19 & 23.29 & 2.26\\
    0.50 & 12.50 & 0.58  & 22.50 & 1.09 & 6.82 & 1.99 \\
    1.00 & 25.00 & 0.57 & 29.30 & 1.14 & 19.13 & 2.06 \\
    \bottomrule
    \end{tabular}
    \caption{Instability under \emph{Safety Attack}($\delta=255$) and the certification for this attack.}\vspace{-0.7cm}
    \label{tab:class_safety}
\end{table*}
\noindent\begin{table*}[h]
\small 
    \centering
    \begin{tabular}{c|c|c|c|c|c|c}
    \toprule
    Noise & \multicolumn{2}{c}{Carla} & \multicolumn{2}{c}{RTK} & \multicolumn{2}{c}{RobotCar} \\ 
    \cmidrule{2-7}
    $\sigma$  
    & \emph{Instability} & \emph{Certificate}  
    & \emph{Instability} & \emph{Certificate}  
    & \emph{Instability} & \emph{Certificate}  \\
    
    \midrule
    0.25 & 29.41 &  0.61 & 28.13 & 0.57 & 27.64 & 0.55\\
    0.50 & 29.41 &  1.19 & 28.10 & 1.06 & 28.12 & 1.11\\
    1.00 & 29.42 &  2.26 & 28.22 & 1.84 & 28.01 & 1.92\\
    \bottomrule
    \end{tabular}
    \caption{Efficiency(Velocity)($\delta=255$) under \emph{Efficiency Attack} and the certification for this attack}\vspace{-0.9cm}
    \label{tab:class_efficiency}

\end{table*} \noindent In the rest of this section, we separately discuss the \emph{Road condition Classification} and \emph{Road Friction Regression} problems. 
For each of the two problems, we start by showing the performance of the non-robust system $\mathcal{A}(\zeta(f)$ in the unmodified environment, and we show the vulnerability of this non-robust  system in malicious environments.
Next, we show the efficacy of the certified robust system $\mathcal{A}(\zeta(g))$ across malicious environments. We empirically show that this certified robust system $\mathcal{A}(\zeta(g))$ ensures the car drives safely and efficiently in malicious driving environments.\noindent\begin{table}[h]\small
    \centering
    \begin{tabular}{c|c|c|c}
    \toprule
       & \emph{Carla} & \emph{RTK} & \emph{RobotCar} \\
    \midrule
    Accuracy & $0\%$    & $80\%$   &   $69\%$ \\
    Instability (\emph{SA}) & 200.00  & 37.50     & 61.50  \\
    Velocity (\emph{EA})&  27.36   & 27.83      & 25.35 \\
    \bottomrule
    \end{tabular}
    \caption{Vunerability of the non-robust perception model $f$, the numbers are the accuracy and performance of $f$ under \emph{PGD} attacks with the adversarial radius $\delta = 255$.}\vspace{-0.8cm}
    \label{tab:class_vul_accuracy}
\end{table}
\noindent\subsubsection{Road Condition Classification}
The perception model takes the driving frame as input in the classification problem and predicts the weather or road types. 
Next, this predicted class is converted to a range of cornering stiffness by referring to Table \ref{tab:cornering2}. 
Table \ref{tab:class_accuracy} shows the accuracy of the non-robust perception model $f$ without malicious attacks. 
Table \ref{tab:class_vul_accuracy} shows the velocity and instability of the car driving in the unmodified environment,
where the attacker doesn't modify the environment. 

\paragraph{Vulnerability}
The first question we ask is \emph{Is perception model $f$ vulnerable to malicious attacks?}, and to this, we answer \emph{yes}. 
The attacker attacks a classifier by flipping the predicted to another label by adding malicious noises to the input image. 
Without loss of generality, we use a common attack, \emph{PGD} attack \cite{maddern20171} as the malicious attacking approach.
Table \ref{tab:class_vul_accuracy} shows the accuracy of the perception model in the maliciously modified environment. 
We see that the accuracy of classification accuracy dropped significantly under the attacks. 
Table \ref{tab:class_vul_accuracy} shows the velocity and deviation in the malicious environment. 
We observe that the accuracy of the classification model $f$, and correspondingly the efficiency and stability of the driving system, drops significantly in the presence of malicious attacks. 
We find system $\mathcal{A}(\zeta(f))$ is indeed vulnerable to malicious attacks. Now, we discuss the robustness of the robust system $\mathcal{A}(\zeta(g))$. 
We will empirically show the effectiveness of the robust perception model $g$ when defending against \emph{Stability Attacks} and \emph{Efficiency Attacks},
and provides the certificate of this robust model. \noindent\begin{table}[t]\small
    \centering
    \begin{tabular}{c|c}
    \toprule
         Label & Safe Set \\
    \midrule
         Sunny       & Sunny, Heavy Rain, Light Rain, Snow\\
         Light Rain  & Light Rain, Heavy Rain, Snow \\
         Heavy Rain  & Heavy Rain, Snow \\
         Snow        & Snow \\
         Asphalt     & Asphalt, Cobblestone, Sand \\
         Cobblestone & Cobblestone, Sand \\
         Sand        & Sand \\
    \bottomrule
    \end{tabular}
    \caption{Safety Class Set.}\vspace{-1cm}
    \label{tab:class_labels_safe}
\end{table} \noindent\paragraph{Certified Robustness}
We start with looking at the results of defending against \emph{Stability Attacks}.\noindent\begin{table*}[h]\small
    \centering \vspace{0.24cm}
    \begin{tabular}{c|c|c|c|c|c|c}
    \toprule
    Noise & 
    \multicolumn{2}{c}{Carla}   & \multicolumn{2}{c}{RTK} & \multicolumn{2}{c}{RobotCar} \\ 
    \cmidrule{2-7}
    $\sigma$          
    & \emph{Instability} & \emph{Velocity}
    & \emph{Instability} & \emph{Velocity}
    & \emph{Instability} & \emph{Velocity} \\
    \midrule
    0.25 &     0.0  &   16.78       &   0.0    &   15.97     & 0.0 & 16.00\\
    0.50 &     0.0 &   16.79        &   0.0    &   15.93    & 0.0 & 15.99\\
    1.00 &     0.0   &   16.65      &   0.0 &   15.83      & 0.0 &
    15.89\\
    \bottomrule
    \end{tabular}
    \caption{ Robustness : instability and efficiency }\vspace{-1cm}
    \label{tab:reg_stability}
\end{table*}
We start with discussing the robustness to stability attacks.
The attacker aims to increase the velocity by modifying the driving frame.
The driving system takes the modified driving frame as input and predicts a high and unsafe velocity. 
Specifically, in the classification problem, the attacker aims to flip the predicted label to a class corresponding to a higher cornering stiffness. 
The control algorithm $\mathcal{A}$ takes this incorrect range of cornering stiffness as input and controls the car at a dangerous speed.
In such a case, the car deviates from its safe trajectory significantly. 

As a defender, we want the car to drive safely in the malicious environment. 
To achieve this goal for different weathers or road types, we defined the safe set for each label in Table \ref{tab:class_labels_safe}. 
For example, the prediction of a corrupted \emph{rain} image could be \emph{snow}, yet not \emph{sunny}, to satisfy the safety criteria. 
A model $g$ is robust if the prediction from $g$ is in the safe set, given a corrupted image $x+\delta$. 
Table \ref{tab:class_safety} shows the efficacy of the robust model $g$ for safety guarantee. 
The numbers in the table are the instability measures. The smaller the number is, the more stable the driving system is.  

We conduct ablation analysis on different Gaussian noises $\sigma$ being added to the smooth function $g$. 
Combining Table \ref{tab:class_safety} and Table \ref{tab:class_vul_accuracy}, we observe that
(1) in \emph{carla dataset}, $\sigma=0.25$ is the best in terms of defending against \emph{Stability Attack}. 
The robust driving system decreases the instability from 200 (shown in Table \ref{tab:class_vul_accuracy}) to 6.36. 
(2) \emph{RTK dataset} is the least vulnerable dataset to malicious attacks, however, the attacker still increases the instability from 0.0 to 37.50. 
The robust driving system decreases this instability to 18.99 when $\sigma=0.25$. 
(3) in \emph{RobotCar dataset}, the robust model deceases the instability from 61.50 to 6.82 when $\sigma=0.5$.

After discussing the performance of the robust model $g$ in the malicious environment, we now show the certification of this robust model. 
Given a sensory input $x$, the smoothed perception model $g$ guarantees the predictions will be within a defined set of labels, if the attack is less than a radius $tau$. This certificate $tau$ is computed via randomized smoothing techniques.
In practice, as in \cite{cohen2019certified}, we apply Monte Carlo process to get an empirical bound. The exact values of these empirical bounds across different datasets are shown in Table 
\ref{tab:class_safety}.

Now we discuss the robustness to efficiency attacks
In this case, the attacker aims to decrease the car's velocity. 
Thus the car may drive unnecessarily cautious under this type of attack. 
Recall the result in Table \ref{tab:class_vul_accuracy}, this type of attack significantly hurts the driving efficiency. 
Specifically, the average speed across different weathers in \emph{Carla Dataset} drops from around 28 to 13. 
As a defender, we want to have the car driving efficiently meanwhile safely, i.e., a relatively high yet safe velocity.
In practice, the defender aims to have the same prediction with and without malicious attacks. 
In other words, the robust model $g$ is not effected by the malicious attacks, formally, $g(x+\delta) = f(x)$. 
Table \ref{tab:class_efficiency} shows the efficacy of the robust model. 
Comparing Table \ref{tab:class_efficiency} and  Table
\ref{tab:class_vul_accuracy}
We observe that the efficiency of $\mathcal{A}(\zeta(g))$ is increased by using $g$. 
Lastly, Table also \ref{tab:class_efficiency} gives the certificate of the defense strategy. 

\subsubsection{Road Friction Regression}
We use a \emph{ResNet-}style regression model. 
Specifically, we modify a \emph{ResNet50} classification model to a regressor by taking the convolutional layers in the classification model, 
and combining it with a linear support vector regression (\emph{SVR}) model. 
We take the weights of the convolutional layers from the trained classification model,
and use \emph{transfer learning} train the parameters in \emph{SVR}. 
We use datasets, \emph{Carla, RTK, Robotcar}, mentioned above. 
Recall that each image in these datasets corresponds to a class. 
This class contains weather and road-type information. 
We convert each class to a corner stiffness by referring to Table \ref{tab:cornering2}.
In particular, we use the mean of the corner stiffness interval in Table \ref{tab:cornering2} as a class's ground truth corner stiffness.\noindent\begin{table}[h]\small
    \centering
    \begin{tabular}{c|c|c|c}
    \toprule
             & \emph{Carla} & \emph{RTK} & \emph{RobotCar} \\
    \midrule
    MSE &     0.008  &     0.017       & 0.016 \\
    Efficiency(\emph{EA}) &     16.82  &     16.01       & 16.03 \\
    Instability(\emph{SA}) &     0.0  &     0.0       & 0.0 \\
    \bottomrule
    \end{tabular}
    \caption{The mean squared error (MSE) and driving performance of non-robust road friction regression in a benign environment.}\vspace{-0.7cm}
    \label{tab:MSE_no_attack}
\end{table}
\begin{table}[h]\small
    \centering
    \begin{tabular}{c|c|c|c}
    \toprule
             & \emph{Carla} & \emph{RTK} & \emph{RobotCar} \\
    \midrule
    MSE &     0.44  &  0.43     & 0.45 \\
    Efficiency(\emph{EA}) &     16.32  &     15.64       & 15.04 \\
    Instability(\emph{SA}) &     200.00  &     200.00       & 200.0 \\
    \bottomrule
    \end{tabular}
    \caption{The mean squared error (MSE) and driving performance of non-robust road friction regression in adversarial environment (\emph{PGD} attack with $\delta=255$.).
    }\vspace{-1.2cm}
    \label{tab:MSE_attack}
\end{table} 
\noindent\paragraph{Vulnerability} We measure the vulnerability of $f$.
Table \ref{tab:MSE_no_attack} shows the mean square error and performance of $f$ without any attacks. 
Table \ref{tab:MSE_attack} shows the mean square error and performance of $f$ under \emph{PGD} attack. 
From these two tables, we observe $f$ is malicious to adversarial attacks, as the MSE increases and performance decreases significantly.


\paragraph{Certified Robustness} At last, we evaluate the robustness of the $h$. 
Table \ref{tab:reg_stability} show the performance of the robust driving system. 
By looking at these tables, we find that combining the certified robust regression model $h$ with 
the robust control algorithm $\mathcal{A}$ guarantees the stability and efficiency in the malicious environment.

%% file: conclusion.tex
\section{Conclusion}
We are the first work combining certified robustness of predictions concerning input adversarial perturbations and robust control. 
We evaluate our proposed approach by applying it to adaptive vehicle control and 
empirically show our approach significantly increases the stability and efficiency of a self-driving car compared with the non-robust baseline counterpart in the malicious environment.
